\documentclass{article}

\PassOptionsToPackage{numbers, compress}{natbib}
\usepackage[preprint]{neurips_2023}

\usepackage[utf8]{inputenc} % allow utf-8 input
\usepackage[T1]{fontenc}    % use 8-bit T1 fonts
\usepackage{hyperref}       % hyperlinks
\usepackage{url}            % simple URL typesetting
\usepackage{booktabs}       % professional-quality tables
\usepackage{amsfonts}       % blackboard math symbols
\usepackage{nicefrac}       % compact symbols for 1/2, etc.
\usepackage{microtype}      % microtypography
\usepackage{xcolor}         % colors
\usepackage{graphicx}
\usepackage{algorithm}
\usepackage{algorithmic}
\usepackage{multirow}

\newcommand{\eg}{\emph{e.g.}}
\newcommand{\ie}{\emph{i.e.}}
\newcommand{\xhdr}[1]{\vspace{1.7mm}\noindent{{\bf #1.}}}

\usepackage{amsmath}
\usepackage{amssymb}
\usepackage{mathtools}
\usepackage{amsthm}
\theoremstyle{plain}
\newtheorem{theorem}{Theorem}[section]

\theoremstyle{definition}

\theoremstyle{remark}

\usepackage{enumitem}
\newcommand{\std}[1]{\fontsize{8}{8}\selectfont$\pm${#1}}
\def\shownotes{1}  %set 1 to show author notes
\ifnum\shownotes=1
\newcommand{\kaidi}[1]{{{\textcolor{blue}{[Kaidi: #1]}}}}

\newcommand{\shir}[1]{{{\textcolor{black}{#1}}}}
\newcommand\jure[1]{{\color{red}\{\textit{#1}\}$_{jure}$}}

\else
\newcommand{\kaidi}[1]{}
\newcommand\jure[1]{}
\fi

\newcommand{\ours}{{CoFree-GNN}}
\newcommand{\drop}{{DropEdge-$K$}}
\newcommand{\reweight}{{Degree-Aware Reweighting}}

\usepackage{amsmath,amsfonts,bm}

\def\eqref#1{equation~\ref{#1}}
\def\1{\bm{1}}

\def\vh{{\bm{h}}}

\def\vm{{\bm{m}}}

\def\vw{{\bm{w}}}
\def\vx{{\bm{x}}}
\def\vy{{\bm{y}}}

\def\mU{{\bm{U}}}

\def\mW{{\bm{W}}}

\DeclareMathAlphabet{\mathsfit}{\encodingdefault}{\sfdefault}{m}{sl}
\SetMathAlphabet{\mathsfit}{bold}{\encodingdefault}{\sfdefault}{bx}{n}

\def\gE{{\mathcal{E}}}

\def\gG{{\mathcal{G}}}

\def\gV{{\mathcal{V}}}

\def\gX{{\mathcal{X}}}

\newcommand{\R}{\mathbb{R}}

\DeclareMathOperator*{\argmin}{arg\,min}

\title{Communication-Free Distributed GNN Training\\
with Vertex Cut}

\author{
Kaidi Cao\textsuperscript{\rm 1},
Rui Deng\textsuperscript{\rm 1},
Shirley Wu\textsuperscript{\rm 1},
Edward W Huang\textsuperscript{\rm 2},
Karthik Subbian\textsuperscript{\rm 2},
Jure Leskovec\textsuperscript{\rm 1}\\
\textsuperscript{\rm 1}Stanford University,
\textsuperscript{\rm 2}Amazon}

\begin{document}

\maketitle

\begin{abstract}

Training Graph Neural Networks (GNNs) on real-world graphs consisting of billions of nodes and edges is quite challenging, primarily due to the substantial memory needed to store the graph and its intermediate node and edge features, and there is a pressing need to speed up the training process. A common approach to achieve speed up is to divide the graph into many smaller subgraphs, which are then distributed across multiple GPUs in one or more machines and processed in parallel. However, existing distributed methods require frequent and substantial cross-GPU communication, leading to significant time overhead and progressively diminishing scalability.
Here, we introduce \ours{}, a novel distributed GNN training framework that significantly speeds up the training process by implementing communication-free training. The framework utilizes a Vertex Cut partitioning, \ie, rather than partitioning the graph by cutting the edges between partitions, the Vertex Cut partitions the edges and duplicates the node information to preserve the graph structure.
Furthermore, the framework maintains high model accuracy by incorporating a reweighting mechanism to handle a distorted graph distribution that arises from the duplicated nodes.
We also propose a modified DropEdge technique to further speed up the training process.
Using an extensive set of experiments on real-world networks, we demonstrate that \ours{} speeds up the GNN training process by up to 10 times over the existing state-of-the-art GNN training approaches.

\end{abstract}

\section{Introduction}

\begin{figure}[h]
    \centering
    \includegraphics[width=0.98\textwidth]{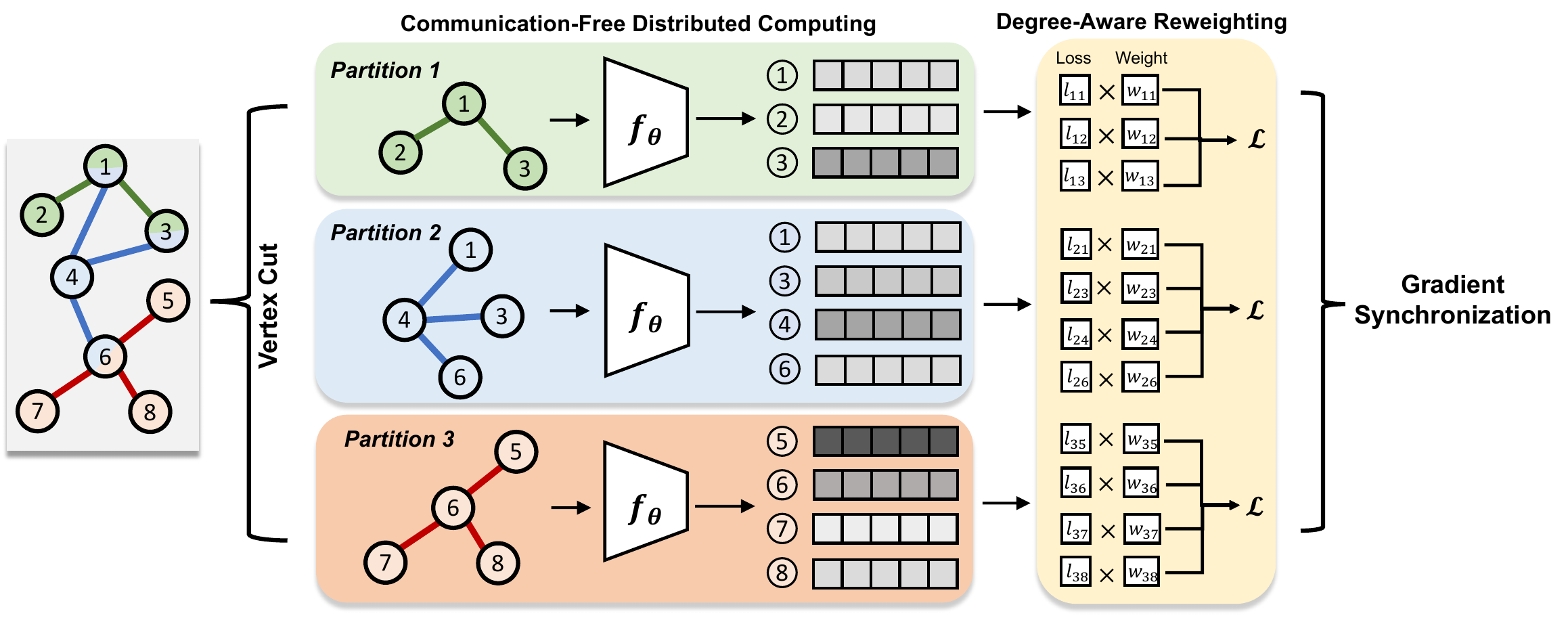}
    %\vspace{-10pt}
    \caption{Overview of \ours{} framework. We begin by dividing the initial large graph into smaller subgraphs using a Vertex Cut partitioning algorithm. Each computational node receives a subgraph and carries out its operations without needing any communication between GPUs. The gradients, which are weighted based on importance, are then gathered to update the weights.}
    \label{fig:approach}
    %\vspace{-15pt}
\end{figure}

Scaling the training of Graph Neural Networks (GNNs)~\citep{kipfsemi,hamilton2017inductive} on large graphs has became a crucial problem~\citep{liu2021exact,wang2019deep}.
A traditional GNN training pipeline generally fits the entire graph including the adjacency matrix and features into a single GPU.
However, since GPUs offer a limited amount of memory, training becomes very challenging for large-scale graphs, hindering the deployment of GNNs in real-world applications, such as social networks~\citep{hamilton2017inductive}, knowledge graphs~\citep{wang2018acekg}, and recommender systems~\citep{ying18rec, lightgcn}. 

Large-scale GNN training has been implemented by taking advantage of large CPU memory as well as multiple GPUs~\citep{shao2022distributed}.
In general, the input graph is partitioned or mini-batch sampled to obtain smaller subgraphs, such that subgraphs can fit into a memory of a single GPU~\citep{hamilton2017inductive,chiang2019cluster,zenggraphsaint}. Then, the model is trained \shir{with a data parallel scheme} on multiple GPUs, where each GPU processes its designated subgraph. 

Distributed training of GNNs is a challenging task, since graph data presents an additional layer of complexity over other data formats due to the interconnections between nodes and edges.
Specifically, since different partitions are connected by edges, dropping these cross-partition edges results in information loss about the graph structure.
To avoid this information loss, distributed GNN training methods require frequent communication between GPUs.
For example, prior approaches support the cross-partition connectivity by synchronizing node embeddings between partitions~\citep{zheng2020distdgl,wanpipegcn,wan2022bns}. 

However, such communication often introduces a major overhead and results in a significant increase in training time. Various strategies have been developed to minimize the required communication, such as advanced compression techniques~\citep{zhang2023boosting}, selective feature exchanges~\citep{zheng2020distdgl,wan2022bns}, and optimized communication protocols that prioritize critical data transfers~\citep{wanpipegcn}. Despite the introduction of these methods, communication overhead continues to be a major obstacle in achieving high scalability for GNN training, especially when the number of partitions scales. 

In this work, we propose \ours{}, a fully communication-free distributed GNN training pipeline (Figure~\ref{fig:approach}). By ``communication-free'', we mean that, given a subgraph, the training process on each GPU is self-sufficient and doesn't require any learned embeddings from the other GPUs.
%We can thus reuse data parallel pipelines implemented in all common deep learning frameworks.
The key idea of \ours{} is to adopt the Vertex Cut graph partitioning method~\cite{PowerGraph}. Specifically, Vertex Cut allocates edges to different GPUs and constructs subgraphs based on the assigned edges on each GPU. Then, \ours{} regards the subgraphs as isolated partitions and trains the GNN model based on the gathered gradients, during which there is no data swap between CPUs and GPUs. Compared to Edge Cut~\citep{zheng2020distdgl} that divides nodes among the GPUs and removes cross-partition edges, Vertex Cut \shir{preserves all the information} about the graph structure since each node or edge is present in at least one partition.

While \ours{} allows for duplicated nodes between different partitions which has some impact on the computational cost, we show that \ours{} is faster than the existing state-of-the-art methods, thanks to its communication-free nature. 

The duplicated nodes across different partitions may distort the graph distribution, \eg, node degree distribution, which could cause bias and hinder the model's ability to generalize to unseen test graphs.
To address this bias, we perform a theoretical examination of the process of gathering gradients from different Vertex Cut partitions. Based on these theoretical insights, we introduce a reweighting mechanism called \reweight{} (DAR) to minimize the bias on gradients caused by the uneven duplication of nodes. Therefore, the reweighted gradients obtained from isolated subgraphs accurately reflect \shir{the training paradigm of Empirical Risk Minimization (ERM) on the full graph.}

Lastly, we integrate our method with a modified DropEdge technique, which further shortens the training time.

% What are the results
We present comprehensive experiments and ablation studies which consistently demonstrate the advantages of \ours{} in both training efficiency and accuracy. 
Specifically, we evaluate the proposed \ours{} on four large real-world datasets: Reddit, Yelp, ogbn-products, and ogbn-papers100M. 
To highlight, compared to the previous state-of-the-art methods, \ours{} achieves a speed increase of up to 10x, while ensuring that the model performance remains on par. Furthermore, the inherent property of \ours{} to operate without the need for communication enhances its scalability,
%and flexibility,
especially as the number of partitions increases.

\section{Related Works}

\xhdr{Distributed GNN Training} Significant research has focused on different aspects of distributed computing techniques to speed up GNN training on large graphs, addressing challenges related to reducing communication, computation, or memory requirements.
Prior works like ROC~\citep{jia2020improving}, NeuGraph~\citep{ma2019neugraph}, and AliGraph~\citep{zhu2019aligraph} have adopted a strategy of partitioning large graphs and storing all the partitions in the CPU memory. 
However, researchers observed that the training efficiency of these methods is hindered by costly data exchanges between the CPUs and GPUs~\citep{wan2022bns}.
Many follow-up works then aim to reduce communication overhead. For example, DistDGL~\citep{zheng2020distdgl} introduces a min-cut graph partitioning algorithm along with multiple balancing constraints and static balancing of the computations. $P^3$~\cite{P3} combines model and data parallelism to avoid transferring features, while assuming the feature dimension is larger than hidden embeddings.
Recently, SAR~\citep{mostafa2022sequential} offers memory reduction for full-graph training; however, it imposes a higher computational load as a result of graph rematerialization.
Finally, PipeGCN~\citep{wanpipegcn} hides the communication overhead by pipelining inter-partition communication with intra-partition computations and BNS-GCN~\citep{wan2022bns} introduces boundary sampling to exchange subsets of neighboring nodes between GPUs, thus substantially reducing the time compared to communicating the complete sets of neighbors. 
While prior work concentrates on reducing the substantial communication overhead during GNN training, our approach here is fundamentally different as it adopts Vertex Cut~\cite{PowerGraph} to restore the full-graph training paradigm without necessitating any communication between GPUs.

\xhdr{Sampling-Based GNN Training}
Another line of work to scale-up GNN training consists of sampling-based training methods: 
(1) Node sampling methods~\citep{hamilton2017inductive, ying18rec, chen18stoc} consider a fixed number or some other subset of neighbors in the message-passing process. However, with the increasing number of message-passing layers, the computational graphs for nodes grow exponentially, requiring exponentially larger resources. 
(2) Layer sampling methods ~\cite{chenfastgcn, zou2019layer, huang18ada} mitigate the disadvantages of the node sampling methods by sampling the $k$-hop neighbors ($k\geq 1$) for a given node independently, thus maintaining a fixed neighborhood expansion factor. 
And, (3) Subgraph sampling methods~\cite{chiang2019cluster, zenggraphsaint, GNNAutoScale, mvs} construct a mini-batch based on the nodes of a sampled subgraph to conduct GNN optimization.

For example, 
Cluster-GCN~\citep{chiang2019cluster} partitions the graph into non-overlapping subgraphs, each containing nodes with high connectivity.

By processing each subgraph separately, memory usage is decreased, which helps with scalability. 
Additionally, gradient bias correction methods have been explored in GraphSAINT~\citep{zenggraphsaint} which uses subgraph sampling with a combination of node, edge, and subgraph sampling strategies. The authors proposed a normalization technique to correct for sampling biases, which enables the model to generalize better on unseen data. Similar techniques for bias correction have been studied in FastGCN~\citep{chenfastgcn}. Sampling-based methods are complementary to our \ours{}, since sampling strategies can be performed on individual partitions. 
Sampling-based methods are complementary to \ours{}, since sampling strategies can be performed on individual partitions. 
\section{Preliminaries}
\xhdr{Graph Neural Networks (GNNs)} In this work, we mainly consider graph node classification tasks following previous works~\cite{chiang2019cluster, zenggraphsaint}.

We define an input graph as $\gG = \{\gV, \gE \}$, where $\gV = \{v_1, v_2, ..., v_n \}$ is the node set and $\gE \subseteq \gV \times \gV$ is the edge set. We use $D(v_i)$ to represent the degree of a node $v_i$. We denote the node features as $\gX = \{\vx_i\in \R^{d}\mid i=1,\ldots,n \}$ and the corresponding node labels as $\vy = \{y_i\in \R^{C} \mid  i=1,\ldots,n \}$. 

Generally, a GNN takes graph-structured data as input and generates embedding vectors for every node in the graph.
Each GNN layer carries out two primary operations: message passing and neighbor aggregation. These processes can be expressed as follows:
\begin{align*}
    \vh_v^{(l)} = \textsc{Agg}^{(l)}\Big(\{\textsc{Msg}^{(l)}(\vh_u^{(l-1)}), u \in \mathcal{N}_\gG(v)\}, \vh_v^{(l)}\Big),
\end{align*}
where $\vh_v^{(l)}$ represents the node embedding after $l$ iterations, while $\mathcal{N}_\gG(v)$ signifies the immediate neighbors of node $v$. The abbreviation $\textsc{Agg}$ refers to the aggregation function, and $\textsc{Msg}$ stands for the message function.
A popular instance of GNNs is GraphSAGE~\citep{hamilton2017inductive} with a mean aggregator:
\begin{align*}
    \vh_v^{(l)} = \mU^{(l)}\textsc{Concat}\Big(\textsc{Mean}\Big(\{\textsc{ReLU}(\mW^{(l)}\vh_u^{(l-1)}),u \in \mathcal{N}_\gG(v)\}\Big), \vh_v^{(l-1)}\Big),
\end{align*}
where $\mU$ and $\mW$ are learnable weight matrices.

\xhdr{Vertex Cut Partitioning}
A Vertex Cut~\cite{PowerGraph} partitioning of a graph $\gG$ involves dividing the edge set $\gE$ into $p$ disjoint partitions, denoted by ${\gG[i] = ( \gV[i], \gE[i] )}$, where each $\gE[i]$ is a subset of $\gE$ and the union of all subsets equals $\gE$. Formally,
\vspace{-5pt} 
\begin{align*}
\gE[i]\cap \gE[j]=\emptyset\ \text{ for any $i,j\leq p$ and $i\neq j$},\ \ \  {\textstyle\bigcup\limits_{i= 1}^{p}}\gE[i]=\gE.
\end{align*}
Note that nodes could be replicated across the Vertex Cut partitions to maintain the information about the graph structure.
%, which may result in extra memory usage and computational overhead.
Replication Factor (RF) then quantifies the amount of replication:
\begin{equation}
    \text{RF}(\gG[1], ..., \gG[p]) = \frac{1}{|\gV|} \sum_{i=1}^p |\gV[i]|.
\end{equation}
    
$\text{RF}(\gG[1], ..., \gG[p])$ can be viewed as an approximation of the computational overhead of the Vertex Cut partition. The goal of Vertex Cut is to minimize this metric to improve efficiency. We also define the replication factor for a single node $v_j$ in $\gG$ as $\text{RF}(v_j) = \sum_{i=1}^p \mathbb{I}(v_j \in \gV[i])$, where $\mathbb{I}$ is an indicator function.  
Various Vertex Cut methods have been proposed to find good partitions~\cite{karypis1997metis, NE, xie2014distributed, hep}. For example, Neighbor Expansion~\citep{NE} (NE) is a Vertex Cut partitioning method that maximizes edge locality in a graph. The main idea is to start with a small set of edges and iteratively expand the set by adding neighboring edges until a balanced partition is achieved. NE has been shown to outperform other state-of-the-art methods on several large graphs. Thus, we adopt NE by default. 

\xhdr{Edge Cut Partitioning} The Edge Cut~\citep{zheng2020distdgl} partitioning divides the node set $\gV$ into $p$ disjoint subsets, \ie, $\{\gV[1], \ldots, \gV[p]\}$. Each partition takes the corresponding node set and constructs the edge set $\gE[i]=\{(v_k, v_m)\mid v_k \in \gV[i], v_m \in \gV[i] \}$. The edges connecting
nodes from different partitions are usually discarded.

\section{Methodology}

We introduce our framework \ours{} in this section. Specifically, we analyze the effect of duplicated nodes in Vertex Cut on performance in Section~\ref{sec:leadin}. We discuss the impact of the changed graph distribution on the model accuracy in Section~\ref{sec:rf}, and we propose a solution to address this change using gradient reweighting in Section~\ref{sec:reweight}. Moreover, in Section~\ref{sec:expedit}, we integrate DropEdge technique into our method to further improve the training speed. We summarize the training procedure of \ours{} in Algorithm~\ref{alg:ours}.

\subsection{Lead-in: From Edge Cut to Vertex Cut} 
\label{sec:leadin}
While Edge Cut serves as the primary large-scale graph partitioning approach in popular graph machine learning packages~\citep{fey2019fast, wang2019deep} and is widely combined with the METIS algorithm~\citep{karypis1997metis}, one of its significant drawbacks is that it discards edges between partitions. These discarded edges can lead to significant information loss about the graph structure, \ie, the neighborhood information of the nodes with discarded edges is partially missing, which can degrade the model performance.

To tackle this challenge, the concept of halo nodes~\cite{zheng2020distdgl} has been introduced to maintain the vital connection information across various partitions. 
These are essentially copies of the boundary nodes from other partitions of the graph, providing the necessary information for updating the nodes within their own partition. Synchronization is required after each iteration to ensure that the halo nodes' information is up-to-date across all GPUs.
Thus, Edge Cut partition with halo nodes still requires the communication between GPUs.
In principle, employing a distributed training pipeline on the combination of Edge Cut and halo nodes (with communications), the model can incorporate all the graph information and it is essentially the same as the full graph training pipeline.

Our primary observation here is that by employing a standard data parallel training pipeline\footnote{The training data is partitioned across the replicas, typically in a distributed computing setup such as multiple GPUs or machines. Each replica performs forward and backward propagation on its own subset of the data, allowing for parallel computation. The gradients from each replica are synchronized and combined to update the shared model parameters.} on the combination of both original nodes and halo nodes, the model theoretically has access to all nodes and edges during training. Consequently, this suggests that such a standard training paradigm is well suited to speed up distributed GNN training and may allow us to avoid the intrinsically slow communication.
We thus adopt a standard data parallel training pipeline
%, which is typically used to speed up training on large data when the data is too large to fit on a GPU,
over a collection of partitioned graphs.  % \jure{we never explain what is "standard data parallel training pipeline". Let's give an explanation when we first mention "data parallel"}
To maintain information about graph structure, some nodes need to be duplicated. As duplicated nodes require more computation which negatively impacts the training speed, we aim to reduce their number.

The following theorem demonstrates that solely using Vertex Cut is always more advantageous than using Edge Cut with halo nodes.

\begin{theorem}
\label{th:cut}
Given an Edge Cut with $H$ halo nodes, any Vertex Cut that adheres to the same partition boundary will exhibit a number of duplicated nodes that is strictly less than $H$. 
\end{theorem}

See the detailed proof in Section~\ref{app:proof}. This theorem indicates that replacing an Edge Cut plus halo nodes with a Vertex Cut partition with duplicated nodes leads to more efficient distributed training.

\subsection{Problem: Imbalanced Replication Factor Hinders Communication-free Training} 
\label{sec:rf}

While Vertex Cut reduces the number of duplicated nodes versus Edge Cut, the Vertex Cut partitions introduce an imbalanced duplication factor between the nodes.
For example, a high degree node $v_j$ could potentially have a higher duplication factor $\text{RF}(v_j)$ compared to nodes with lower degrees, which distorts the node feature frequency from the original graph distribution and could yield suboptimal model performance.
In other words, due to the duplicated nodes, Vertex Cut partitions may over-represent certain high-degree nodes while under-representing others, creating a distribution gap between the original node feature distribution and the distribution after partitioning. Formally, we encapsulate the issue of imbalanced replication factor into the following theorem:
% \rok{What is the issue and how is it encapsulated?}

\begin{theorem}
Let use consider a graph $\gG$ with no isolated node, where the degree of each node $D(v_j)$ in the graph $\gG$ follows a power-law degree distribution. Given that we divide the original graph into $p$ partitions, the replication factor of node $v_j$, i.e., $\text{RF}(v_j)$ presents an imbalance ratio greater than or equal to:
$$\frac{1 - (1 - \frac{1}{p})^{\max_j D(v_j)}}{1 - (1 - \frac{1}{p})^{\min_j D(v_j)}}$$
Here, $\max_j D(v_j)$ and $\min_j D(v_j)$ denote the maximum and minimum node degrees in graph $\gG$, respectively. 
\end{theorem}

Under the common assumption that the node degrees follow a power-law distribution~\citep{PowerGraph},
this theorem emphasizes the degree of imbalance in the replication factor when partitioning the graph, which creates a distribution gap between the partitions and the original graph.

The distribution gap could lead to crucial implications in the GNN distributed training. Specifically, when applying Empirical Risk Minimization (ERM) on Vertex Cut partitions, the training is essentially performed on a less accurate representation of the graph, which may negatively impact the final accuracy of the model.

\subsection{Solution: Reweighting Gradients to Recover Full Graph Training Paradigm} % Recover Full Graph Training Paradigm
\label{sec:reweight}

To address the distribution gap, a na\"ive way is to synchronize node embeddings between GPUs. However, this approach introduces  communication overhead which reduces the training speed. Inspired by the previous work~\cite{zenggraphsaint}, we propose \reweight{} (DAR) to reweight the gradients from different Vertex Cut partitions to recover the full-graph training paradigm. The intuition is to counterbalance the distribution gap caused by the different node replication factors. Notably, the reweighting mechanism is easy to implement in a distributed training pipeline. 

In order to determine the appropriate weights for reweighting the gradients, we formulate an optimization problem. Specifically, the general idea is to find a weight assignment that minimizes the Frobenius norm of the difference between the gradients calculated from the original graph and those from the graph partitions.
The optimization problem can be formulated as follows:
\begin{equation}
    \vw^* = \argmin _{\vw} | \nabla_\theta \sum_{v_j \in \gV} \mathcal{L} (\vh_j^L, y_j) - \sum_{i=1}^p \nabla_\theta \sum_{v_j \in \gV[i]} \vw_{ij} \cdot \mathcal{L} (\vh_j^{L}[i], y_j)|_F.
\end{equation}
Here, $\gV$ and $\gV[i]$ denote the set of nodes in the original graph and the $i$-th partition, respectively. $\vh_j^L$ and $\vh_j^L[i]$ represent the model's final-layer prediction on the node $v_j$ of the original graph and the $i$-th partition, respectively.
The term $\mathcal{L} (\vh_j^L, y_j)$ is the loss 
calculated using the true labels $y_j$ and the model prediction $\vh_j^L$. The variable $\vw_{ij}$ denotes the weights assigned to the loss of node $v_j$ in the $i$-th partition. This objective aims to minimize the discrepancy between the gradients of the original graph and the weighted gradient computed from the individual partitions. 

To obtain the optimal weight assignment, we observe that it can be achieved by an approximation, \ie, $\vw^*_{ij} =  {D(v_j[i])}/{D(v_j)}$. Specifically, let us consider a single GraphSAGE layer ($L=1$).
We can prove the following theorem:

\begin{theorem}
Let $f_\theta : \gG \rightarrow \mathbb{R}^C$ be a single-layer GraphSAGE model with sigmoid activation and binary cross-entropy loss $\ell(\cdot, \cdot)$. Let there be an undirected, unweighted
graph $\gG = (\gV, \gE)$, together with its Vertex Cut partition $\{\gG[1], ..., \gG[p] \}$, where $\gG[i] = (\gV[i], \gE[i])$. Assuming graph $\gG$ is homophilic, then for a distributed training algorithm running on all partitions, optimizing the following reweighted loss recovers the full graph training paradigm:

\begin{equation}
    \label{eq:reweight}
    \mathcal{L}(f_\theta, \gG[i]) = \sum_{v_j[i] \in \gV[i]}  \frac{D(v_j[i])}{D(v_j)} \cdot \ell(\vh_j[i], y_j).  
\end{equation}
\end{theorem}

We include a detailed proof in Section~\ref{app:proof}. 

By incorporating these optimal weights and optimizing the loss function in Equation~\ref{eq:reweight}, the model takes into account the discrepancies between node replication factors and can be more accurately trained on Vertex Cut partitions. 
Thus, we effectively address the imbalanced replication factor in Vertex Cut partitioning, while eliminating the need for communications between GPUs. For multi-layer GNNs, we found that the optimal weights in the above formulation work well empirically (see results in Section~\ref{sec:experiments}).

\subsection{Further Optimizations for Speed}
\label{sec:expedit}

We further demonstrate the flexibility of \ours{} in combining other graph training techniques. 

One such technique is DropEdge~\citep{rongdropedge}, which alleviates the bias caused by graph partitioning according to the following theorem:

%and integrates an extra regularization term into the distributed optimization.

\begin{theorem}
DropEdge effectively alleviates the bias resulting from graph partitioning and introduces an additional regularization term in the distributed optimization process.
\end{theorem}

Nonetheless, directly applying DropEdge to large graph partitions may adversely affect the runtime, as the edge sampling operation can occasionally take more time than the backward propagation itself. To tackle this challenge, we suggest an alternative approach called \drop{}. The concept behind \drop{} is to preprocess $K$ DropEdge masks and then, during each iteration, randomly select and apply one of these masks.

By using the \drop{} approach, we efficiently address the runtime issue associated with naively implementing DropEdge on large graph partitions. This method allows for faster training, as the masks are preprocessed and readily available for use in each iteration. Consequently, the overall runtime is reduced, and the training process becomes more efficient.

\begin{algorithm}[h]
    \caption{Training a GNN with \ours{}}
    \label{alg:ours}
    \begin{algorithmic}[1]
    \REQUIRE A GNN $f$ parameterized by $\theta$. Full graph $\gG=\{\gV, \gE\}$. Number of partitions $p$. Learning rate $\lambda$. A hyperparameter \textsc{use\_DropEdge} to enable DropEdge-$K$.
    \STATE Partition $\gG$ into $\gG[1], ..., \gG[p]$ via Vertex Cut
    \STATE Assign each partition to a GPU node
    \STATE \texttt{// for the $i$-th partition (per partition view)}
    \STATE $\vw^*_i\leftarrow \{j: D(v_j[i]) / D(v_j)\mid v_j \in \mathcal{V}[i]\}$
    \STATE $\{M_i^{(k)}\}_{k=1}^K\leftarrow $ Preprocessed DropEdge masks on $\gE[i]$
    \WHILE{not converge}
    \IF{\textsc{use\_DropEdge}}
    \STATE Randomly select one of $\{M_i^{(k)}\}_{k=1}^K$ and apply the mask on $\gE[i]$
    \ENDIF
    \STATE Compute $\mathcal{L}(f_\theta, \gG[i])$ via Equation~\ref{eq:reweight}
    \STATE Update model parameters: $\theta\leftarrow\theta - \lambda \nabla_{\theta}\mathcal{L}(f_\theta, \gG[i])$
    \ENDWHILE
    \end{algorithmic}
\end{algorithm}

\section{Experiments}
\label{sec:experiments}
\subsection{Experimental Setup}

\xhdr{Datasets} 
We evaluate our proposed model using four diverse datasets: Reddit~\citep{hamilton2017inductive}, Yelp~\citep{zenggraphsaint}, ogbn-products, and ogbn-papers100M~\citep{hu2020open}, each with distinct characteristics.

\textbf{Reddit} contains user-generated content for natural language processing and sentiment analysis tasks. It has 233 thousand nodes (posts) and 114 million edges (interactions).
\textbf{Yelp} consists of business reviews and ratings for sentiment analysis and recommendation systems. The dataset includes approximately 716 thousand nodes (reviews) and 7 million edges (business-user relationships).
\textbf{ogbn-products} represents a graph of Amazon products connected by co-purchasing links for node classification and link prediction tasks. It contains over 2 million nodes (products) and 62 million edges (co-purchases).
\textbf{ogbn-papers100M} includes scientific papers and their citation relationships for node classification and citation prediction tasks. The dataset features over 111 million nodes (papers) and 1.6 billion edges (citations).
For all datasets, the evaluation is performed using the node classification task. The Reddit, ogbn-products, and ogbn-papers100M datasets are assessed using the accuracy metric following the previous works~\cite{hu2020open,chenfastgcn}, while the Yelp dataset is evaluated based on the micro-F1 score following GraphSAINT~\cite{zenggraphsaint}.

\xhdr{Baselines} To demonstrate the effectiveness and improvement on training speed of \ours{} in handling large-scale graph data, we compare its performance with two classes of methods:
\begin{itemize}[leftmargin=*]
    \item \textbf{Sampling-based baselines}: \textbf{GraphSAGE}, \textbf{Cluster-GCN}, and \textbf{GraphSAINT}. These methods offer different approaches to node sampling and aggregation for training GNNs on large-scale graphs. GraphSAGE uses neighborhood sampling to construct several sub-graphs to formulate a batch. Cluster-GCN partitions the graph into clusters to reduce the computational complexity. At the training phase, each batch is formulated by a random set of subgraphs. GraphSAINT employs a normalization strategy to improve the robustness of the training process. 
    \item \textbf{Distributed GNN training frameworks}:
    %In addition to these sampling-based techniques, we also benchmark our method against distributed GNN training frameworks, such as
    \textbf{DistDGL}, \textbf{PipeGCN}, and \textbf{BNS-GCN}. DistDGL is a distributed implementation of the Deep Graph Library (DGL) that scales GNNs across multiple GPUs and machines with a balanced edge min-cut partitioning algorithm. PipeGCN proposes to hide communication overhead by pipelining inter-partition communication with intra-partition communication, while BNS-GCN proposes boundary sampling to significantly reduce communication cost.
    %By comparing our method against these state-of-the-art baselines, we aim to demonstrate its effectiveness and scalability in handling large-scale graph data.
\end{itemize}

\xhdr{Other Settings}  We implemented \ours{} with PyTorch~\citep{paszke2017automatic} and DGL~\citep{wang2019deep}. We conduct the experiments with NVIDIA A100 GPUs, each has 80GB HBM2e Memory, two 64-thread CPUs (AMD EPYC 7543 32-Core Processor) and 2.0T DDR4 Memory. The intra-server communication (CPU-GPU and GPU-GPU) is based on PCIe 4.0 lanes. 
If not explicitly mentioned, we by default use NE~\citep{NE} algorithm for Vertex Cut. We use 10 masks for \drop{} and the DropEdge ratio is set to 0.5. We will make source code public at the time of publication.
Please refer to additional implementation details in Appendix~\ref{app:implementation}.

\subsection{Empirical Comparisons}
\label{sec:empirical_comparisons}
Table~\ref{tab:main_runtime} illustrates the per-iteration runtime of our model on three separate datasets: Reddit, ogbn-products, and Yelp. Concurrently, Table~\ref{tab:main_acc} provides a comprehensive overview of the final test accuracy. Generally, methods based on sampling tend to yield lower final accuracy.
Specifically, DistDGL operates on multiple GPUs, but within each GPU, it continues to use several samplers and train sampled subgraphs, which introduces additional runtime overhead.
In a broader context, \ours{} habitually matches or exceeds the accuracy of others, while boasting a substantial throughput increase of \shir{1.7x to 16.8x} against BNS-GCN. The throughput of \ours{} escalates more effectively as we increase the number of partitions, largely due to the elimination of embedding communication during the forward phase. The only synchronizations needed across GPUs are gradients of the weights, which are considerably smaller than the node features, especially for large graphs.

Additionally, by incorporating the \drop{} method into \ours{}, computation can be further accelerated. Among the three datasets, the Reddit dataset exhibits denser edge connections, resulting in greater computational efficiency compared to the other datasets.

\begin{table}[t]
\centering
\caption{The average runtime (\textit{milliseconds}) per iteration and the standard deviation across 10 trials on Reddit, ogbn-products, and Yelp datasets. Time Reduced Factor indicates the range of improvement on runtime compared with the baselines. \ours{} 
 improves runtime without sacrificing test accuracy (also see Figure~\ref{fig:converge_reddit}) compared to the full-graph training paradigm.  %\jure{Just to understand: In theory the runtime should be scaling linearly. That is if iteration on 2 GPUs takes X, then on 4 GPUs it shouldt ake X/2, right?}
 }
\label{tab:main_runtime}
\resizebox{1\textwidth}{!}{
\begin{tabular}{l|rrrrrr}
\toprule
& \multicolumn{2}{c}{Reddit} & \multicolumn{2}{c}{ogbn-products} & \multicolumn{2}{c}{Yelp} \\
\#Partitions/GPUs  & 2            & 4           & 5               & 10              & 3           & 6          \\
\midrule
DistDGL     &    863.9\std{37.3}   & 746.9\std{42.7}      &                2879.5\std{78.1}   & 2418\std{81.6}               &       2143.7\std{53.9}   &      1893.6\std{47.4}        \\
PipeGCN     &  316.8\std{3.2}           &  393.7\std{3.3}           &   851.8\std{2.4}      &   771.6\std{1.8}             &     356.0\std{3.6}       &     358.6\std{2.4}     \\
BNS-GCN     &   314.0\std{1.9}         &     305.3\std{2.6}      &       740.9\std{3.9}        &    717.6\std{4.1}      &   569.1\std{1.7}        &   551.3\std{3.2}  \\
\ours{}    &   168.8\std{1.3}           &     96.3\std{1.7}        &        72.0\std{0.8}       &        47.4\std{0.8}           &   211.2\std{2.1}         &   143.4\std{1.6}   \\
\ours{}+\drop{}  &  116.5\std{2.2}  & 76.6\std{1.3} & 65.3\std{0.7} & 42.7\std{0.9} & 205.3\std{1.6} & 137.1\std{1.5} \\
\midrule
Time Reduced Factor& $2.7\sim7.4$ &$4.0\sim9.8$ & $11.3\sim44.1$& $16.8\sim56.6$&$1.7\sim10.5$&$2.6\sim 13.8$\\
\bottomrule
\end{tabular}}
\end{table}

\begin{table}[t]
\centering
\caption{The average accuracy and standard deviation across 10 trials for test performance on Reddit, ogbn-products, and Yelp datasets. \ours{} does not sacrifice test accuracy compared to the full-graph training paradigm with partitioned graphs.}
\label{tab:main_acc}
\resizebox{1\textwidth}{!}{
\begin{tabular}{l|rr|rr|rr}
\toprule
& \multicolumn{2}{c|}{Reddit} & \multicolumn{2}{c|}{ogbn-products} & \multicolumn{2}{c}{Yelp} \\
\midrule
GraphSAGE   &            \multicolumn{2}{c|}{95.67\std{0.03}}            &               \multicolumn{2}{c|}{78.81\std{0.32}}              &          \multicolumn{2}{c}{63.51\std{0.39}}          \\
Cluster-GCN &           \multicolumn{2}{c|}{95.92\std{0.05}}               &                \multicolumn{2}{c|}{78.91\std{0.40}}                   &             \multicolumn{2}{c}{61.32\std{0.47}}              \\
GraphSAINT  &           \multicolumn{2}{c|}{96.63\std{0.04}}           &              \multicolumn{2}{c|}{79.11\std{0.29}}              &              \multicolumn{2}{c}{65.30\std{0.31}}        \\
\midrule
\midrule
\#Partitions  & 2            & 4           & 5               & 10              & 3           & 6          \\
\midrule
DistDGL     &     96.51\std{0.02} &      96.28\std{0.02}     &     79.11\std{0.39} & 79.01\std{0.46}            &     65.06\std{0.05}  & 65.02\std{0.04}          \\
PipeGCN     &  97.12\std{0.02}            &    97.04\std{0.03}           &      79.06\std{0.42}             &        78.91\std{0.65}         &          65.27\std{0.01}   &       65.24\std{0.02}      \\
BNS-GCN     &   97.14\std{0.02}           &     97.12\std{0.03}        &       79.42\std{0.17}          &          79.27\std{0.23}      &  65.27\std{0.02}           &  65.30\std{0.02}   \\
\ours{}    &   97.12\std{0.02}           &    97.10\std{0.02}         &         80.28\std{0.04}        &   80.40\std{0.03}               &      65.28\std{0.01}       & 65.34\std{0.02}    \\
\ours{}+\drop{} & 97.14\std{0.02} & 97.11\std{0.02} & 80.31\std{0.03} & 80.43\std{0.02} & 65.33\std{0.01}  & 65.36\std{0.02} \\
\bottomrule
\end{tabular}}
\end{table}

We assess \shir{our method on a multi-node training scenario} by examining its performance on the ogbn-papers100M dataset.
The original graph is divided into 192 segments, and the training is performed on three machines, each equipped with eight GPUs. A summary of the results can be found in Figure~\ref{fig:papers100m}. Our findings demonstrate that, in contrast to previous state-of-the-art methods like PipeGCN and BNS-GCN, that spend most of the training time on communication, \ours{} manages to decrease the training time per iteration by tenfold.

\subsection{Ablation Studies}
\xhdr{Scaling the number of partitions} In Section~\ref{sec:empirical_comparisons}, we demonstrated the training speed of \ours{} scales up well. % due to its adoption of a standard data parallel training pipeline. 
Here we further show the scalability of our model by assessing the training time for each epoch while adjusting the partition quantity. The corresponding results are depicted in Figure~\ref{fig:scalability}. The figure indicates that by doubling the number of partitions, the training speed nearly doubles across all three datasets.
In addition, we demonstrate that \ours{} can maintain a performance level that is comparable to the full graph paradigm as the number of partitions increases. To validate this point, we gradually increased the number of partitions to 256 (see Figure~\ref{fig:partitions} in Appendix). These experiments were carried out in a simulated environment by accumulating gradients.
The findings indicate that our method consistently sustains the training accuracy, even when the number of partitions is increased to a substantial number like 256.

\xhdr{Empirical analysis on convergence} To study the effect on convergence of the proposed \ours{} and original full graph training, we visualize training/validation curve per epoch in Figure~\ref{fig:converge_reddit}. We show that the convergence speed in terms of epochs of \ours{} is quite similar to the full graph training paradigm.
%\rok{What is a full graph training paradigm? It is not listed under the baselines. How was it implemented? It should be described in the Related Works.}

\xhdr{Ablation study on the reweighting schema} In order to assess the efficacy of our suggested DAR scheme, we evaluated it against two basic alternatives: (1) none, which implies no node-wise reweighting is performed, and (2) vanilla-inv, which implies each node is weighted by the inverse of its replication factor $1 / {\text{RF}(v_j)}$, without taking into account edges. As demonstrated by the outcomes of the ablation study displayed in Table~\ref{tab:reweight}, the DAR scheme exhibits greater accuracy compared to its less complex counterparts.

\xhdr{Ablation study on partition algorithms} Our full pipeline \ours{} is agnostic about specific partition algorithms. We incorporated an ablation study focusing on various graph partition algorithms.
\shir{Due to space limit, we place the results in Table~\ref{tab:partitionchoice} (Appendix \ref{app:exp}) and discuss the main findings here. 
Our observations reveal that the Edge Cut partition (METIS) falls short in performance compared to Vertex Cut partitions. The differences in performance among other partitioning methods are fairly slight across various partitioning methods and datasets.}

\begin{figure}[t]
\begin{minipage}[b]{0.32\linewidth}
\centering
\includegraphics[width=\textwidth]{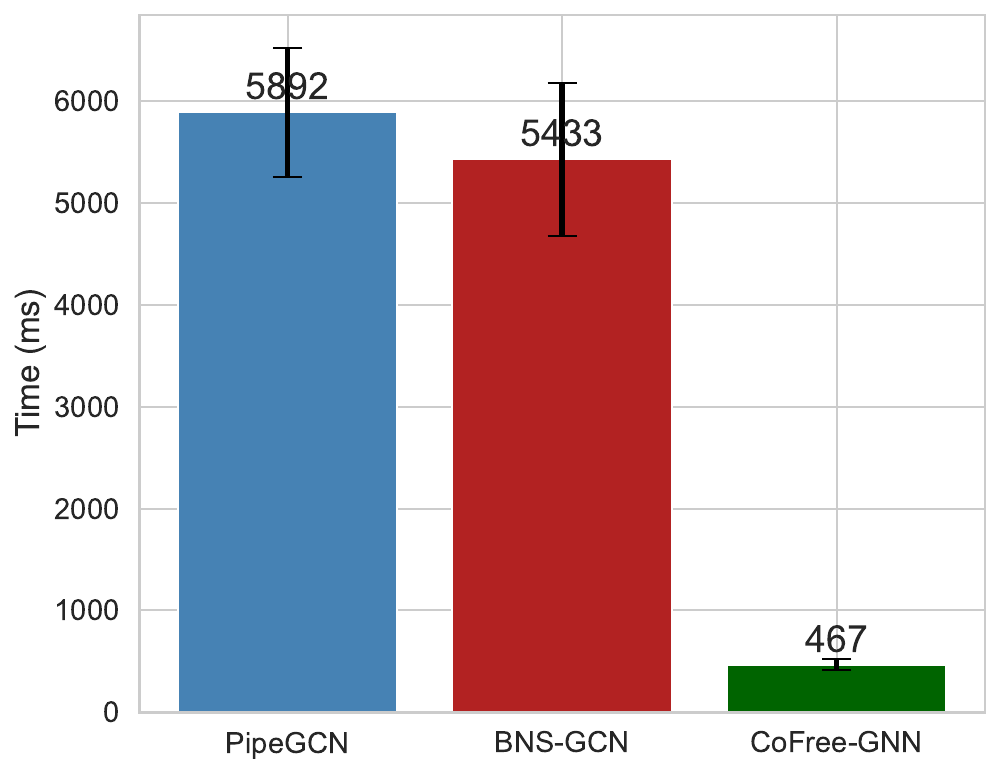}
\caption{The average runtime per iteration on ogbn-papers100M dataset under multi-node training scenario. \ours{} manages to decrease the training time by tenfold.}
\label{fig:papers100m}
\end{minipage}
\hspace{0.1cm}
\begin{minipage}[b]{0.32\linewidth}
\centering
\includegraphics[width=\textwidth]{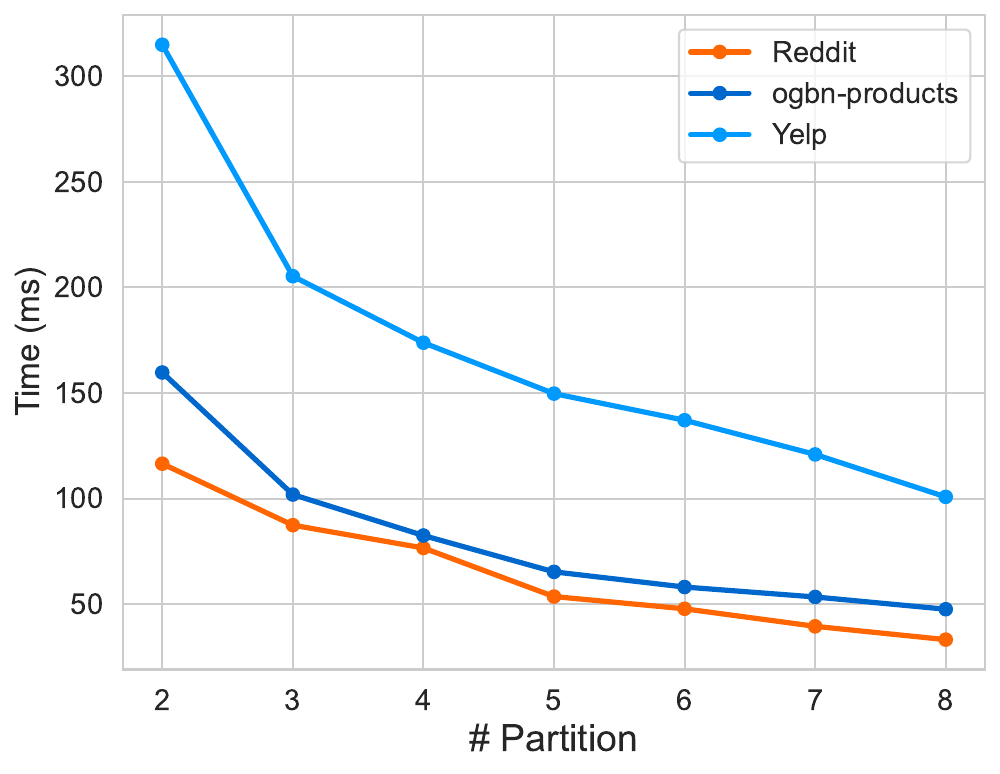}
\caption{The effect of scaling number of partitions on training time per iteration. By doubling the number of partitions, the training speed nearly doubles across all three datasets.}
\label{fig:scalability}
\end{minipage}
\hspace{0.1cm}
\begin{minipage}[b]{0.32\linewidth}
\centering
\includegraphics[width=\textwidth]{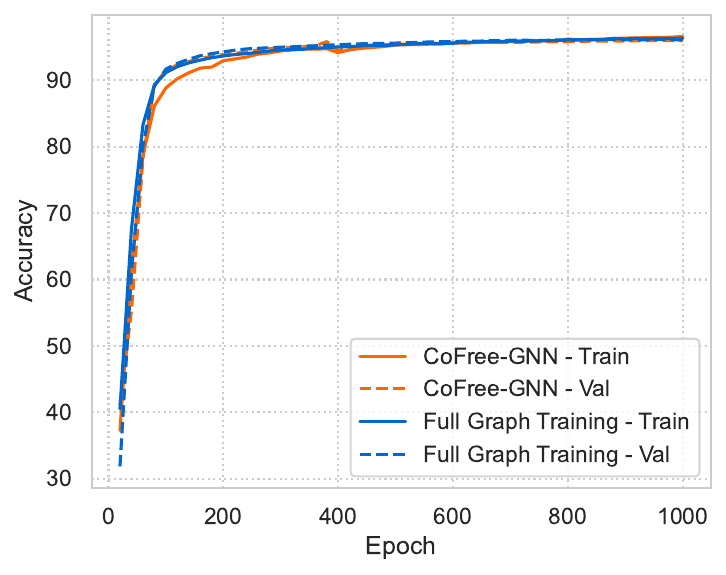}
\caption{Training curves of \ours{} and full graph training on Reddit dataset. The convergence speed \textit{v.s.} epochs of \ours{} is similar to the full graph training paradigm.}
\label{fig:converge_reddit}
\end{minipage}
\end{figure}

\begin{table}[t]
\centering
\caption{Ablation study on reweighting schema: The average accuracy and standard deviation across 10 trials on 256 partitions' setting. DAR outperforms other alternatives in terms of final test accuracy.}
\label{tab:reweight}
% \fontsize{9}{9}\selectfont
\begin{tabular}{l|ccc}
\toprule
% Reweighting schema
& Reddit &  ogbn-products & Yelp \\
             \midrule
    w/o reweighting          &   96.20\std{0.02}        &     78.45\std{0.04}            &   64.76\std{0.02}              \\
    vanilla-inv          &   96.83\std{0.02}        &     79.27\std{0.04}            &   64.96\std{0.02}              \\
     DAR      &     97.08\std{0.02}      &          80.26\std{0.03}                & 65.39\std{0.02}  \\
\bottomrule
\end{tabular}
\end{table}

\section{Conclusion}

In this study, we addressed the challenge of training Graph Neural Networks (GNNs) on large graphs by developing \ours{}, a fully communication-free distributed GNN training pipeline. \ours{} uses the Vertex Cut graph partitioning rather than the traditional Edge Cut partitioning, which reduces information loss and enables communication-free training.
Our method is enhanced by reweighting gradients to minimize the impact of uneven node duplication and by using a modified DropEdge method to further accelerate the training process.
Evaluations on multiple datasets demonstrate \ours{}'s superiority in training efficiency and accuracy, achieving up to a 10-fold reduction in training time while maintaining test performance. Our communication-free approach also \shir{maintains training speed} as the number of partitions increases.

\shir{Moreover, we discuss the limitations of our work in Appendix \ref{app:limitations}, such as increased memory usage due to the duplication of node information across partitions.}
Further research could extend our proposed methodology to a wider range of graph-based tasks and explore new techniques to reduce computational overhead. Overall, the findings from our study provide a promising direction for achieving higher training speed in GNNs, and we anticipate that our work will spur further advancements in this field.

\section*{Acknowledgements}
KC and JL acknowledge the support of
DARPA under Nos. HR00112190039 (TAMI), N660011924033 (MCS);
ARO under Nos. W911NF-16-1-0342 (MURI), W911NF-16-1-0171 (DURIP);
NSF under Nos. OAC-1835598 (CINES), OAC-1934578 (HDR), CCF-1918940 (Expeditions), 
NIH under No. 3U54HG010426-04S1 (HuBMAP),
Stanford Data Science Initiative, 
Wu Tsai Neurosciences Institute,
Amazon, Docomo, GSK, Hitachi, Intel, JPMorgan Chase, Juniper Networks, KDDI, NEC, and Toshiba.
The content is solely the responsibility of the authors and does not necessarily represent the official views of the funding entities.

\bibliography{main}
\bibliographystyle{plainnat}

\newpage
\appendix

\section{Missing Proofs and Derivations}
\label{app:proof}
\begin{theorem}
Given an Edge Cut with $H$ halo nodes, any Vertex Cut that adheres to the same partition boundary will exhibit a number of duplicated nodes that is strictly less than $H$. 
\end{theorem}
\begin{proof}
Let the selected edges for the Edge Cut partition be represented by $\gE' \subset \gE$. Consequently, the set of duplicated nodes corresponds to the vertices in $\gE'$, denoted as $\gV'$. The total count of halo nodes generated by the Edge Cut partition is given by $|\gV'|$. To transition from an Edge Cut partition to a Vertex Cut partition, we can randomly eliminate nodes from $\gV'$ along with their connected edges and assign the edge to the node left. It is evident that at least one node must be removed, allowing for a Vertex Cut with a reduced quantity of duplicated nodes.
\end{proof}

\begin{theorem}
Let use consider a graph $\gG$ with no isolated node, where the degree of each node $D(v_j)$ in the graph $\gG$ follows a power-law degree distribution. Given that we divide the original graph into $p$ partitions, the replication factor of node $v_j$, \ie, $\text{RF}(v_j)$ presents an imbalance ratio greater than or equal to:
$$\frac{1 - (1 - \frac{1}{p})^{\max_j D(v_j)}}{1 - (1 - \frac{1}{p})^{\min_j D(v_j)}}$$
Here, $\max_j D(v_j)$ and $\min_j D(v_j)$ denote the maximum and minimum node degrees in graph $\gG$, respectively. 
\end{theorem}

\begin{proof}
Let us first consider a randomized vertex cuts algorithm that randomly assigns each edge to one of the $p$ partitions. The expected replication $\mathbb{E}[\text{RF}(v_j) ]$ can be computed by considering the process of randomly assigning the $D(v_j)$ edges adjacent to $v_j$. Let the indicator $X_i$ denote the event that the vertex $v_j$ has at least one of its edges on the machine $i$. The expectation is then
\begin{align*}
    \mathbb{E}[\text{RF}(v_j) ] &= p \cdot (1  - \mathbf{P}(v_j \text{ has no edge on machine } i)) \\
    &= p \cdot (1 - (1 - \frac{1}{p})^{D(v_j)})
\end{align*}
Thus, the maximum imbalance ratio among all nodes is $\frac{1 - (1 - \frac{1}{p})^{\max_j D(v_j)}}{1 - (1 - \frac{1}{p})^{\min_j D(v_j)}}$.

More advanced vertex cut algorithms are designed by enforcing the heuristics that cutting higher-degree vertices can result in fewer mirror vertices, \eg, \citep{xie2014distributed,zhang2017graph}. Thus, high-degree nodes tend to have even more duplications, and the imbalance ratio of more advanced vertex cut algorithms will only be higher.
\end{proof}

\begin{theorem}
Let $f_\theta : \gG \rightarrow \mathbb{R}^c$ be a single-layer GraphSAGE model with sigmoid activation and binary cross-entropy loss. Let there be an undirected, unweighted
graph $\gG = (\gV, \gE)$, together with its Vertex Cut partition $\{\gG[1], ..., \gG[p] \}$, where $\gG[i] = (\gV[i], \gE[i])$. Assuming graph $\gG$ is homophilic, then for a distributed training algorithm running on all partitions, optimizing the following reweighted loss recovers the full graph training paradigm.
\begin{align*}
    \mathcal{L}(f_\theta, \gG[i]) = \sum_{v_j[i] \in \gV[i]}  \frac{D(v_j[i])}{D(v_j)} \cdot \ell(\vh_j[i], y_j)
\end{align*}
\end{theorem}

\begin{proof}
First, since the graph is assumed to exhibit homophily, and all the partitions $\{ \gG[1], \gG[2], ..., \gG[p] \}$ do not contain isolated nodes, the embedding of each node does not deviate significantly from its embedding in the full graph:
\begin{align*}
\mathbf{h}^L_{j} \approx \mathbf{h}^L_{j} [i] \text{ for } i = 1, \dots, p
\end{align*}
Thus,
\begin{align*}
\nabla_{\mathbf{h}_j^L} \mathcal{L} (\mathbf{h}_j^L, y_j) \approx \nabla_{\mathbf{h}_j^L[i]} \mathcal{L} (\mathbf{h}_j^{L}[i], y_j)
\end{align*}
As each edge is not duplicated, we study the gradient flow along the edge $(k,j)$.
For the original graph, we have 
\begin{align*}
    \nabla_{\vh_k^{L-1}} \mathcal{L} (\vh_j^L, y_j) &= \nabla_{\vh_k^{L-1}} \vh_j^L \times \nabla_{\vh_j^L}   \mathcal{L} (\vh_j^{L}, y_j) \\
    &= \frac{1}{D(v_j)} (\mathbb{I}(\mW^{L-1}\vh_k^{L-1} > 0) \mU^{L-1} \cdot \mW^{L-1}) \times \nabla_{\vh_j^L}   \mathcal{L} (\vh_j^{L}, y_j) 
\end{align*}
For the Vertex Cut partition, assuming edge $(k,j)$ lies in partition $\gG[i]$, we have 
\begin{align*}
    \nabla_{\vh_k^{L-1}[i]} \mathcal{L} (\vh_j^L[i], y_j) &= \nabla_{\vh_k^{L-1}[i]} \vh_j^L[i] \times \nabla_{\vh_j^L[i]}   \mathcal{L} (\vh_j^{L}[i], y_j) \\
    &= \frac{1}{D(v_j[i])} (\mathbb{I}(\mW^{L-1}\vh_k^{L-1}[i] > 0) \mU^{L-1} \cdot \mW^{L-1}) \times \nabla_{\vh_j^L[i]}   \mathcal{L} (\vh_j^{L}[i], y_j) 
\end{align*}
Given the above estimation, we can conclude that when choosing $\vw_{ij} =  \frac{D(v_j[i])}{D(v_j)}$, we have
\begin{align*}
    \nabla_\theta \sum_{v_j \in \gV} \mathcal{L} (\vh_j^L, y_j) \approx \sum_i  \nabla_\theta \sum_{v_j \in \gV[i]} \vw_{ij} \cdot \mathcal{L} (\vh_j^{L}[i], y_j)
\end{align*}
\end{proof}

\begin{theorem}
DropEdge reduces bias induced by graph partitioning and introduces an additional regularization term into the distributed optimization.
\end{theorem}
\begin{proof}
We represent $\vm^{l}_j$ as $\textsc{Msg}^{(l)}(\vh_j^{(l)})$, and introduce the random variable
\begin{align*}
\eta = \begin{cases}
-1 & \text{with probability } p \\
\frac{1}{1-p} & \text{with probability } 1-p 
\end{cases}
\end{align*}
We define $ \delta_j^l \triangleq (\mathbf{1} + \eta) \vm_j^l - \vm_j^l$ as the disturbance in the messages.
To analyze the impact of this disturbance, we use a Taylor expansion around $\delta_j^{l}=0$ on the final loss. In Section A.2 of~\citet{wei2020implicit}, it is suggested that when $ W \cdot \vm_j^{l} \approx 0$, the perturbation of $\delta^{l}_j$ to the loss function may not be significantly large, thus supporting the application of the Taylor Expansion.
\begin{align}
\label{eq:taylor}
&\mathbb{E}_{\delta_j^l} [ \mathcal{L}(\vm_j^l + \delta_j^l) - \mathcal{L}(\vm_j^l) \nonumber ] \\
\approx& \mathbb{E}_{\delta_j^l} [ D_{\vm_j^l} \mathcal{L}(\vm_j^l)\delta_j^l + \frac{1}{2} {\delta^l_j}^\top (D^2_{\vm_j^l} \mathcal{L} (\vm_j^l)) \delta_j^l ]
\end{align}
In Eq.~\ref{eq:taylor}, the expectation over the linear term vanishes because $\delta_j^l$ is a vector with zero mean. Consequently, the second-order term acts as additional regularization.
\end{proof}

\section{Additional Implementation Details}
\label{app:implementation}

\xhdr{Implementation details for Reddit} We consider GraphSAGE with 4-layer and 256 hidden units.
We use Adam optimizer with the base learning rate of 0.01. We train for 3000 epochs and 0.3 dropout rate.

\xhdr{Implementation details for ogbn-products} We consider GraphSAGE with 3-layer and 128 hidden units.
We use Adam optimizer with the base learning rate of 0.003. We train for 500 epochs and 0.3 dropout rate.

\xhdr{Implementation details for Yelp} We consider GraphSAGE with 4-layer and 512 hidden units.
We use Adam optimizer with the base learning rate of 0.01. We train for 3000 epochs and 0.1 dropout rate.

\xhdr{Implementation details for ogbn-papers100M} We consider GraphSAGE with 3-layer and 128 hidden units.
We use Adam optimizer with the base learning rate of 0.01. We train for 100 epochs and 0.5 dropout rate.

\section{Additional Experiments}

\xhdr{Scaling the number of partitions} In Section~\ref{sec:empirical_comparisons}, we demonstrated the training speed of \ours{} scales up well due to its adoption of a standard data parallel training pipeline.
Here, we further demonstrate that \ours{} can maintain a performance level that is comparable to the full graph paradigm as the number of partitions increases. To do this, we gradually increased the number of partitions to 256 (see Figure~\ref{fig:partitions}). These experiments were carried out in a simulated environment by accumulating gradients.
The results show that \ours{} maintains the training accuracy even when scaling up the number of partitions.

\xhdr{Ablation study on partition algorithms} Our full pipeline \ours{} is agnostic about specific partition algorithms. We incorporated an ablation study focusing on various graph partition algorithms.
We find that Edge Cut partition (METIS) underperforms Vertex Cut partitions. The performance gaps of other partitioning methods are relatively minor across the datasets.

\label{app:exp}
\begin{figure}[h]
\centering
\includegraphics[width=0.5\textwidth]{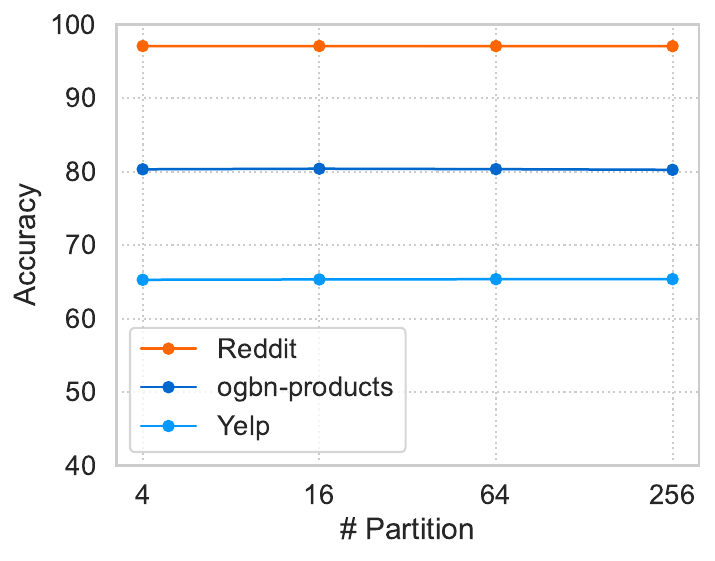}
\caption{Ablation study on the effect of number of partitions on testing accuracy.
}
\label{fig:partitions}
\end{figure}

\begin{table}[htpb]
\centering
\caption{Ablation study on partition algorithms the average accuracy and standard deviation across 10 trials on 256 partitions' setting.} 
\label{tab:partitionchoice}
\begin{tabular}{l|l|ccc}
\toprule
& & Reddit &  ogbn-products & Yelp \\
             \midrule
Edge Cut &  METIS~\cite{karypis1997metis}          &     95.42\std{0.02}      &     78.63\std{0.05}            &   61.01\std{0.04}      \\
Vertex Cut  &  Random          &    96.94\std{0.03}       &      79.88\std{0.07}           &     65.27\std{0.03}            \\
 Vertex Cut  &  NE~\cite{NE}     &  97.08\std{0.02}      &          80.26\std{0.03}                & 65.39\std{0.02}   \\
 Vertex Cut   & DBH~\cite{xie2014distributed}     &    97.06\std{0.02}       &      80.23\std{0.04}                   &  65.33\std{0.02} \\
  Vertex Cut  & HEP~\cite{hep}     &  97.13\std{0.02}         &         80.28\std{0.02}                &  65.40\std{0.02} \\
\bottomrule
\end{tabular}
\end{table}

\section{Broader Impact}
\label{app:impact}
This work presents a novel distributed Graph Neural Network (GNN) training framework called \ours{}, designed to substantially accelerate the training process while minimizing cross-GPU communication. This development could have several significant impacts across various fields and sectors. \textbf{Scalable Applications}: The scalability and adaptability of \ours{} could foster more extensive use of GNNs in applications that handle large-scale graph data. This includes social network analysis, recommendation systems, biological network analysis, and transportation networks, among others. \textbf{Efficient Resource Utilization}: By minimizing inter-GPU communication, \ours{} can enable more efficient use of hardware resources. This could have implications for energy usage and cost-effectiveness in data centers, potentially leading to more sustainable and affordable machine learning solutions. \textbf{Digital Divide}: The benefits of this technology might be unequally distributed, especially in regions with limited access to advanced computational resources necessary for large-scale GNNs. This could potentially exacerbate the digital divide, highlighting the need for equitable access to such resources.

\section{Limitations}
\label{app:limitations}
While our proposed framework, \ours{}, shows great promise in accelerating the training process of Graph Neural Networks (GNNs) and reducing the need for cross-GPU communication, it is important to acknowledge several potential limitations of our work: \textbf{Duplicated Nodes in Vertex Cut Partitioning}: While the Vertex Cut partitioning approach that we employ significantly reduces communication overhead, it may lead to increased memory usage due to the duplication of node information across partitions. This could become problematic when dealing with graphs where node attributes are numerous or large in size. \textbf{Node Duplication and Consistency}: Duplicating nodes across multiple GPUs could potentially lead to issues in maintaining consistency of node information, especially in online or incremental learning scenarios where node attributes are updated frequently.

\end{document}